\documentclass{article}[12pt] 

 \usepackage{amstext, amsthm, dsfont}

   \theoremstyle{plain}
     \newtheorem{theorem}{Theorem}
     \newtheorem{claim}{Claim}
     \newtheorem{lemma}[theorem]{Lemma}

   \theoremstyle{definition}
     \newtheorem{definition}[theorem]{Definition}

   \newcommand{\argmin}{\mathrm{argmin}}

  \newcommand{\Acal}{{\mathcal A}}

  \newcommand{\Pcal}{{\mathcal P}}

  \newcommand{\Ncal}{{\mathcal N}}

  \newcommand{\Xcal}{{\mathcal X}}

\title{Notes on Classes with Vapnik-Chervonenkis Dimension 1}

\author{Shai Ben-David \\
 School of Computer Science\\
University of Waterloo,\\ Waterloo,Ontario, N2L 3G1,
Canada\\
{\tt shai@cs.uwaterloo.ca} }
 
\begin{document}

\maketitle

\begin{abstract}

The Vapnik-Chervonenkis dimension is a combinatorial parameter that reflects the "complexity" of a set of sets (a.k.a. concept classes). It has been introduced by Vapnik and Chervonenkis in their seminal  paper \cite{vapnik1971uniform} and has since found many applications, most notably in machine learning theory and in computational geometry. Arguably the most influential consequence of the VC analysis is the fundamental theorem of statistical machine learning, stating that a concept class is learnable (in some precise sense) if and only if its VC-dimension is finite. Furthermore, for such classes a most simple learning rule - empirical risk minimization (ERM) - is guaranteed to succeed.

The simplest non-trivial structures, in terms of the VC-dimension, are the classes (i.e., sets of subsets) for which that dimension is 1.

In this note we show a couple of curious results concerning such classes. The first result shows that such classes share a very simple structure, and, as a corollary, the labeling information contained in any sample labeled by such a class can be compressed into a single instance.

The second result shows that due to some subtle measurability issues, in spite of the above mentioned fundamental theorem, there are classes of dimension 1 for which an ERM learning rule fails miserably\footnote{I have discovered the results  presented in this note more than 20 year ago, and have mentioned them in public talks as well as private communications
over the years. However, this is the first time I have written them up for publication.}.

\end{abstract}

\section{Preliminaries: The Vapnik-Chervonenkis dimension}
%
\begin{definition} [\cite{vapnik1971uniform}]
Let $\Xcal$ be any set, let $2^{\Xcal}$ denote its power set - the set of all subsets of $\Xcal$.
A \emph{concept class} is a set of subsets of $\Xcal$, $H \subseteq 2^{\Xcal}$.
We will identify subsets of $\Xcal$ with binary valued functions over $\Xcal$ (a function $h: \Xcal \to \{0,1\}$ is identified with with the set $h^{-1}(1)=\{x \in \Xcal: h(x)=1\}$).
\begin{itemize}

\item $H$ \emph{shatters} $A \subseteq \Xcal$ if $\{h \cap A : h \in H\}=2^A$. Note that for a finite $A$ this is equivalent to  $|\{h \cap A : h \in H\}|=2^{|A|}$.
\item The \emph{Vapnik-Chervonenkis dimension} of $H$ is defined as $VCdim(H)=\sup \{|A|: H shatters A\}$.

\end{itemize}
\end{definition}
This note focuses on classes whose VC-dimension is 1. The following simple claim is well known and can be easily verified.

\begin{claim} \label{claim_inclusion}
Given a class $H$ over some domain set $\Xcal$. If there exists a linear order $\preceq$ over $\Xcal$ such that very member $h$ of $H$ is an initial segment w.r.t. that order (namely, for all $x, y \in \Xcal$, if $x \preceq y$ and $h(y)=1$ then $h(x)=1$) 
then $VCdim(H) \leq 1$.
\end{claim}

\begin{definition}
For  functions $h, f: \Xcal \to \{0,1\}$, the \emph{$f$-representation of $h$} is the set $h_f=\{x \in \Xcal: h(x)\neq f(x)$. Note that if $f$ is the constant 0 function then $h_f$ is just the usual set equivalent of the function $h$.
For a class of functions $H$ and $f: \Xcal \to \{0,1\}$, we define the \emph{$f$-representation of $H$} as $H_f=\{h_f : h \in H \}$. 
\end{definition}

\noindent Note that for any concept class $H$, and any binary valued $f$ as above, $VCdim(H)=VCdim(H_f)$.

The VC dimension plays a major role in machine learning theory. We discuss this aspect some more in Section \ref{subsec_more_prem}.

\section{A structure theorem for classes of VCdim 1}

In this section we show that classes of VC-dimension 1 are in fact very simple. We have already mentioned, in Claim \ref{claim_inclusion} that if the sets in a class $H$ are linearly ordered
by inclusion, then $VCdim(H)=1$. This claim can be somewhat extended by noting that one does not really need a \emph{linear} order. In fact, having the inclusion partial ordering of the members 
of $H$ being a \emph{tree} suffices to imply the same conclusion. This is formalized by the following.

\begin{definition}
We say that a partial order $\preceq$ over some set $\Xcal$ is a \emph{tree ordering} if, for every $x \in \Xcal$ the initial segment $I_x=\{y: y \preceq x\}$ is linearly ordered (under $\preceq$).
\end{definition}

\begin{claim} \label{claim_gen_tree}
Given a class $H$ over some domain set $\Xcal$. If there exists a tree ordering $\preceq$ over $\Xcal$ such that very member $h$ of $H$ is an initial segment w.r.t. that order (namely, for all $x, y \in \Xcal$, if $x \preceq y$ and $h(x)=1$ then $h(y)=1$) 
then $VCdim(H) \leq 1$.

\end{claim}

\begin{proof}
The proof of that claim is simple -
\end{proof}

We will now show that any class having VC-dimension 1 has such a structure.

\begin{theorem} \label{thm_structure}
Let $H$ be a concept class over some domain $\Xcal$. The following statements are equivalent:
\begin{enumerate}
\item  $VCdim(H) \leq 1$.
\item There exists some tree ordering over $\Xcal$ and a representation $f: \Xcal \to \{0,1\}$ such that every element of $H_f$ is an initial segment under that ordering relation.
\end{enumerate}
\end{theorem}

\begin{proof}
\noindent \textbf{ 1 implies 2:} 
 Just note that if every member of $H$ s an initial segment under $\preceq$ then, for any $x_1, x_2 \in \Xcal$ if there exists some $h \in H$ such that $h(x_1)=h(x_2)=1$ then it must be the case that either $x_1 \preceq x_2$ or $x_2 \preceq x_1$.
However, in the first case there exist no $h' \in H$ such that $h'(x_1)=0$ and $h'(x_2)=1$ and in the second case there exist no $h' \in H$ such that $h'(x_2)=0$ and $h'(x_1)=1$, therefore the set $\{x_1, x_2\}$ is not shattered by $H$.

\noindent \textbf{ 2 implies 1:} Assume, w.l.o.g., that for every $x \neq y \in X$, there exists some $h \in H$ so that $h(x)\neq h(y)$.
Pick some $f \in H$ and consider the partial ordering $\leq^H_{f}$ defined by

\[\leq^H_f=\{(x, y): \forall h \in H, ~h(y) \neq f(y) \rightarrow
h(x) \neq f(x)\}. \]

Lemma \ref{lem_tree_order} shows that this is indeed a tree ordering. The proof is concluded by noting that the definition of the relation $\leq^H_f$ implies that for every $h \in H$, the set $h_f$ (namely, $\{x: h)x) \neq f)x)\}$) is an initial segment w.r.t $\leq^H_f$. 
\end{proof}

%
%
%

\begin{lemma} \label{lem_tree_order} 
$\leq^H_f$ is a partial ordering. Namely, it is reflexive, transitive and anti symmetric. Furthermore, 
the assumption that $VCdim(H) \leq 1$ implies that 
$\leq^H_f$ is a the ordering.
\end{lemma} 
\begin{proof}

\begin{itemize}

\item Being reflexive and transitive follows trivially from the definition.
\item For anti-symmetry, let $x,y$ be such that both $x \leq^H_f y$ and $y \leq^H_f x$ hold. It is easy to see
that this implies that for all $h \in H$, $h(x)=h(y)$.
\item Assume, by way of contradiction, that $\leq^H_f$ is not a tree ordering. This means that for some $x \in \Xcal$ there exist $y, z$ so that $y \leq^H_f x$, $z \leq^H_f x$ but neither $x \leq^H_f z$ nor $z \leq^H_f y $ holds.
Let us show that in such a case the pair $\{y, z\}$ is shattered by $H$ (and thus $VCdim(H) \geq 2$ contradicting our assumption). Pick $h_1 \in H$ for which $h_1(x) \neq f(x)$ (such $h_1$ exists by our assumetion that for every $x \in \Xcal$, each of the labels $\{0,1\}$ are given by some $h \in H$). The definition of $\leq^H_f$ implies now that $h_1(y) \neq f(y)$ and $h_1(z) \neq f(z)$. The non-compatibility of $y,z$ implies the existence of $h_2, h_3 \in H$ such that $h_2(y)=f(y)$
and $h_2(z) \neq f(z)$ and those labels are flipped for $h_3$. It follows that $\{h_1, h_2, h_3, f \}$ shatter $\{y, z\}$ and since we picked $f \in H$, it follows that $H$ also shatters $\{y,z\}$.
\end{itemize}
\end{proof}



\subsection{Sample compression for classes of Vcdim 1}

The above structure theorem has a nice implication for the issue of \emph{sample compression schemes}.

\begin{definition}
A sample compression scheme of size $d$ for a class $H$ is a pair of functions, $F$, $G$, such that $F$ maps samples $S$ from $ \bigcup_{m \in \Ncal} (\Xcal \times \{0,1\})^m$ to samples $F(S) \in  \bigcup_{0 \leq m \leq d} (\Xcal \times \{0,1\})^m$
such that for any such $S$, if there exists some $h \in H$ that is constant with $S$ (namely, for all $(x,y) \in S$, $h(x)=y$) then $F(S) \subseteq S$, and $G :  \bigcup_{0 \leq m \leq d} (\Xcal \times \{0,1\})^m \to 2^{\Xcal}$ such that for any $S$ and any $(x, y) \in S$, $G(F(S))(x)=y$.\\

A sample compression scheme is called \emph{unlabeled } if for every $G(S)$ consists of just a subset of $\Xcal$ (of elements appearing in $S$), without their labels.

\end{definition}

Sample compression schemes were introduced by Littlestone and Warmuth \cite{Littlestone-Warmuth} and a long standing open problem is the conjecture that there is some content $C$ such that every concept class of finite VC-dimension
has a sample compression scheme of size $C Vcdim(H)$.

 Theorem \ref{thm_structure} readily implies that every class of VC dimension 1 has an unlabeled  sample compression scheme of size 1 as follows:
 
 \begin{quote} Given a class $H$ such that $VCdim(H)=1$, let $f$ be a member of $H$ and $\leq^H_f$ as in the proof of Theorem  \ref{thm_structure}.
 For a sample $S = ((x_1, h(x_1)), \ldots (x_m, h(x_m))$ (for some $h \in H$), let $$F(S)= ~\mbox{the $\leq^H_f$-maximal element in} ~\{x_i: i \leq m ~\mbox{and} ~h(x_i)\neq f(x_i)\}$$
 
(and $F((x_1, h(x_1), \ldots (x_m, h(x_m)))=\emptyset$ if $\{x_i: i \leq m ~\mbox{and} ~h(x_i)\neq f(x_i)\}=\emptyset$).\\

Let $G$ be the function that on input $x$ outputs the function $G(x)$ so that that on input $y \in \Xcal$,  if  $x \leq^H_f y$ then $G(x)(y)=f(y)$, and for any $y$ such that $y \leq^H_f x$, $G(x)(y) = 1-f(y)$ .
\end{quote}

It is easy to verify that the pair $(F,G)$ is a size 1 unlabeled compression scheme for $H$.

\section{ERM may fail to learn VC dimension 1 Classes }

\subsection{More preliminaries} \label{subsec_more_prem}


\noindent \textbf{The probability setup:} For any given domain set $\Xcal$, we will consider probability distributions over $\Xcal \times \{0,1\}$. Given such a probability distribution $P$, we define
the induced labeling rule as the function $\ell_P: \Xcal \to [0,1]$ defined by $\ell_P(x)=P(y=1|x)$, and the induced marginal distribution, $ D_P$,  as the projection of $P$ on $\Xcal$. 
We will identify $P$ with the pair $(D_P, \ell_P)$.

\begin{definition} [0-1 loss]
\begin{enumerate}

\item For a probability distribution $P$ over $\Xcal \times \{0,1\}$ and $h: \Xcal \to \{0,1\}$, 
\[L_{P}(h) = P[\{(x,y) : h(x) \neq y\}] \]

\item For a finite $S \subseteq \Xcal \times \{0,1\}$ and $h: \Xcal \to \{0,1\}$, 
\[L_S(h) = \frac{|\{(x,y) \in S : h(x) \neq y\} |}{|S|}  \]

\end{enumerate}
\end{definition}

\begin{definition} A class $H$, over some domain set, $\Xcal$, has the\emph{ Uniform Convergence Property (UCP)} with respect to a family of probability distributions $\Pcal$ over 
$\Xcal \times \{0,1\}$, if for every $\epsilon >0, \delta >0$ there exist some $m_H(\epsilon, \delta) \in \Ncal$ such that for every $P \in \Pcal$, $m \geq m_H(\epsilon, \delta)$ implies that
\[ \Pr_{S \sim P^m} [\sup_{h \in H}|L_S(h) - L_P(h)| >\epsilon] < \delta.\]

\end{definition}

\noindent \textbf{Learning and Empirical Risk Minimization:} A \emph{learning rule}  is a function that takes labeled samples as input and outputs a classifier. Formally,
it is a function $\Acal : \bigcup_{m \in \Ncal} (\Xcal \times \{0,1\})^m \to 2^{\Xcal}$.
\begin{definition} 

\begin{enumerate}
\item A learning rule $\Acal$  is an \emph{Empirical Risk Minimizer (ERM)} for some class $H$, if $\Acal(S) \in \argmin \{L_S(h) : h \in H \}$, for every $S \in \bigcup_{m \in \Ncal} (\Xcal \times \{0,1\})^m$.
\item A learning rule $\Acal$  is a \emph{Probably Approximately Correct (PAC)} learner for some class $H$ w.r.t. some measurable algebra $(\Xcal, \Omega)$, if for every $\epsilon >0, \delta >0$ there exist some $m_H(\epsilon, \delta) \in \Ncal$ such that for every probability measure $P$ over $(\Xcal, \Omega) \times \{0,1\}$, $m \geq m_H(\epsilon, \delta)$ implies that
\[ \Pr_{S \sim P^m} [\sup_{h \in H}L_P(\Acal(S)) - L_P(h) >\epsilon] < \delta. \]
It is common to omit the $\sigma$-algebra of measurable sets, $\Omega$, from the notation. It is implicitly assumed to be the full power set of $\Xcal$ if $\Xcal$ is finite or countably infinite, or the  Lebesgue $\sigma$-algebra when $\Xcal$ is a subset of some Euclidean space.

\end{enumerate}
\end{definition}
The following claim is well known and can be easily verified.
\begin{claim} A class $H$ has the Uniform Convergence Property with respect to the family of probability distributions $\Pcal$ over 
$(\Xcal,  \Omega) \times \{0,1\}$ if and only if any ERM learning function is a PAC learner for $H$.

\end{claim}

The following is a seminal result that, in a sense spearheaded modern machine learning theory.
\begin{theorem}[Vapnik-Chervonenkis 1971 \cite{vapnik1971uniform}]  
A class $H$ has the uniform convergence property if and only if its Vapnik-Chervonenkis dimension is finite.

\end{theorem}
In their proof, Vapnik and Chervonenkis invoke some subtle measurability assumption. The vast literature of PAC style learning that followed that paper, often fails to mention that assumption. In the next section we show that such a condition is indeed necessary.

\subsection{A class of VCdim 1 for which ERM fails}

Let our domain set be the real unit interval $[0,1]$ and let $U$ be the uniform (Lebsegue) measure over it. 

\begin{theorem} \label{thm:ERMfails} Assuming the continuum hypothesis, there exists a class of VC dimension 1 such that, for some probability distribution over its domain and for some classifier $h \in H$, some empirical risk minimization (ERM) rule fails badly when trained over samples generated by $P$ and labeled by $h$. More concretely, for any sample size $m$ with probability 1 over samples of that size, the error of that rule, when applied to the sample,  will be 1. \\

Furthermore, it is a class of measurable subsets of the unit interval, and the probability distribution with respect to which it fails is the uniform distribution over that interval.

\end{theorem}

\begin{proof} Recall that the continuum hypothesis states that $2^{\aleph_0}=\aleph_1$ (in other words, that every infinite subset of reals is either countable or it can be mapped \emph{onto} the full set of reals). It is well known that this assumption implies (in fact, equivalent to) the existence of a well ordering, $\prec$ over  $[0,1]$, so that every initial segment is countable (for every $r \in [0,1]$, the set $\{s: s \prec r\}$ is countable). Given such an ordering define a class of subsets $H=\{[0,1]\} \cup\{h_r: r \in [0,1]\}$, where, for each real number $r$,
$h_r=\{s: s \prec r\}$.

Note that, by the choice of the ordering relation $\prec$, every set in $H$ is either countable or equals the unit interval. Therefore each member of the class $H$ is Lebesgue measurable.

Furthermore, since $\prec$ is an ordering over the real interval, for every $s \prec t$, $h_s \subseteq h_t$ (and every $h_s$ is a subset of the set $[0,1]$. It follows that 
VC-dim($H)=1$.

We will now show that there is an ERM learning algorithm for the class $H$ that fails badly.
Define the learning rule $\Acal$ as follows:
\begin{quote} Given any finite sample $S= ((r_1, y_1), (r_2, y_2), \ldots (r_m, y_m))$ labeled according to some $h \in H$,
 let $r_S^{\star} = \max\{r_i: (r_i, 1) \in S\}$ and define $\Acal(S) = h_{r_S^{\star}}$.
\end{quote}

Pick  $t=[0,1]$ as the target classifier. That is, the labeling rule that assigned the value 1 to every instance. Every training sample has, therefore, the form  $S=((r_1, 1), (r_2, 1), \ldots (r_m, 1))$. By the above definition of the ERM earning rule we consider,  for any such sample $S$, $L_S(\Acal(S))=0$. However, since $\Acal(S)=h_r$ for some $r \in[0,1]$, and $h_r = \{s: s \prec r\}$,  by our choice of the ordering relation  $\prec$, $\Acal(S)$ is a countable set (that is,  assigns the label 1 only to countably many instances). It follows that, for the uniform distribution, $U$, $L_{U, t}(\Acal(S))=P([0,1] \setminus \Acal(S))=1$ (where, for a marginal probability distribution, $D$,  and labeling rules $t,h  : X \to \{0,1\}$, $L_{D,t}(h) \stackrel{\rm def}{=} D[\{x: h(x \neq t(x) \}]$). In other words, for every sample size, $m$, with probability 1 over $P$-generated i.i.d. samples, $S$, of that size, the $0-1$ loss of $\Acal(S)$ is 1.
\end{proof}

How come the above example does not contradict  the Vapnik-Chervonenkis characterization of ERM learnability in terms of the VC-dimension (a.k.a. the Fundamental Theorem of Statistical Machine Learning)? The devil is, of course, in measurability issues. The common proof of that fundamental theorem goes through the double sample trick. 

To prove the theorem one needs to upper bound the probability of the set of samples for which an ERM learner fails. Namely
\[\Pr_{S \sim P^m} [ \exists h \in H ~\mbox{such that} ~L_S(h)=0 ~\mbox{but} ~ L_P(h)> \epsilon\}]\]
This is usually done by upper bounding that event by 
\[2\Pr_{S, T \sim (P^m)^2} [ \exists h \in H ~\mbox{such that} ~L_S(h)=0 ~\mbox{but} ~ L_T(h)> \epsilon\}]\]

(and, for any $H$ with a finite VC-dimension,  this probability can be shown to go to zero, as $m$ goes to $\infty$ based on Sauer's lemma).

However, for such an argument to go through, it should be the case that last probability exists. Namely, that the set 
\[\Delta^m(H) \stackrel{\rm def}{=}\{S, T \in (X \times \{0,1\})^{2m} : ~\exists h \in H ~\mbox{such that} ~L_S(h)=0 ~\mbox{but} ~ L_T(h)> \epsilon\}\]
is measurable under the product measure $P^{2m}$.

In the case of the example used to prove Theorem \ref{thm:ERMfails}, that last measurability requirement fails already for $m=1$.
To see that, note that 
\[\Delta^1(H)=\{(x,y): x \prec y\}\]
(when we fix the labeling function that labels $S$ and $T$ to be the constant 1 function).

Recall that in the above example, the domain set is $X=[0,1]$ and the underlying probability distribution is the uniform distribution, or equivalently, the Lebesgue measure.

The fact that $\{(x,y): x \prec y\}$ is not measurable under the Lebesgue measure over $[0,1]^2$ follows from the failure of Fubini's integration lemma for the (characteristic function of) that set:\\

\[\int_{y=o}^{1} \int_{x=0}^{1} \large{\mathbf 1}_{x \prec y} dx ~dy = \int_{y=o}^{1} 0 ~dy = 0\]

whereas

\[\int_{x=o}^{1} \int_{y=0}^{1} \large{\mathbf 1}_{x \prec y} dy ~dx = \int_{x=o}^{1} 1~dx = 1\]

The first equation holds since for every $y$, $\{x: x \prec y\}$ is countable, so $\int_{x=0}^{1} \large{\mathbf 1}_{x \prec y} dx = 0$ for any $y$. 
The second equation holds since for every $x$, $\{y: x \prec y\}$ is co-countable, so $ \int_{y=0}^{1} \large{\mathbf 1}_{x \prec y} dy= 1$ for any $x$. 


\bibliographystyle{unsrt}
\bibliography{Researchbib.bib}

\end{document}